\documentclass[latin1]{iaf}
\annee{January 2013}

\newcommand {\rr} {k}
\newcommand {\rf} {\mathit{rank}}

\newcommand {\lingconc} {\mathcal{S}}

\newcommand {\ent} {\mathrel{{\scriptstyle\mid\!\sim}}}

\newcommand {\nott} {\lnot}

\newcommand {\esiste} {\exists}
\newcommand {\sx} {\langle}
\newcommand {\dx} {\rangle}
\newcommand {\incluso} {\subseteq}

\newcommand {\emme} {\mathcal{M}}

\newcommand {\unione} {\cup}

\newcommand {\tc} {\mid}
\newcommand {\vuoto} {\emptyset}
\newcommand {\WW} {\mathcal{W}}

\newcommand{\tip}{{\bf T}}

\newcommand{\alc}{\mathcal{ALC}}
\newcommand{\alct}{\mathcal{ALC}+\tip}

\newcommand{\alctr}{\mathcal{ALC}^{\Ra}\tip}
\newcommand{\modelsalctr}{\models_{\scriptsize\mathcal{ALC}^{\scriptsize{\Ra}}\tip}}
\newcommand{\alctp}{\alct}
\newcommand{\alctrm}{\mathcal{ALC}^{\Ra}_{min}\tip}

\newcommand{\be}{\begin{enumerate}}
\newcommand{\ee}{\end{enumerate}}

\newcommand{\hide}[1]{}

\newcommand{\subsud}{\utilde{\sqsubset}}

\def \cases{\left \{\begin{array}{l}}
\def \endcases{\end{array}\right .}

\newcommand {\CL} {{\bf CL}}
\newcommand {\Cu} {{\bf C}}
\newcommand {\Pe} {{\bf P}}
\newcommand {\Ra} {{\bf R}}

\newcommand {\ri} {\rightarrow}

\newcommand {\bes} {\begin{description}}
\newcommand{\ens} {\end{description}}

\newcommand {\beq} {\begin{quote}}
\newcommand {\enq} {\end{quote}}
\newcommand {\bit} {\begin{itemize}}
\newcommand {\enit} {\end{itemize}}

\newcommand{\mint}{\sqcap}

\newenvironment{proof}{\noindent {\bf Proof.}} {}
\usepackage{umoline}
\usepackage{times}
\usepackage{graphicx}
\usepackage{color}
\usepackage{amssymb}
\usepackage{undertilde}

\newtheorem{theorem}{Theorem}[section]

\newtheorem{proposition}[theorem]{Proposition}
\newtheorem{definition}[theorem]{Definition}
\newtheorem{example}[theorem]{Example}

\titre{On Rational Closure in Description Logics of Typicality}

\auteurs{Laura Giordano$^1$, Valentina Gliozzi$^2$, Nicola Olivetti$^3$, Gian Luca Pozzato$^2$} 

\institutions{
$^1$ DISIT, Universit\`a del Piemonte Orientale, Alessandria, Italy\\
$^2$ Dipartimento di Informatica, Universit\`a degli Studi di Torino, Italy\\
$^3$ LSIS-UMR CNRS 7196 Aix-Marseille Universit\'e, France}

\mels{\normalsize{laura@mfn.unipmn.it, \{gliozzi,pozzato\}@di.unito.it, nicola.olivetti@univ-amu.fr}}

\begin{document}
\creationEntete

\begin{abstract} 
We define the notion of rational closure in the context of Description Logics extended with a tipicality operator.  We start from $\alct$, an extension of $\alc$ with a typicality operator $\tip$: intuitively  allowing to express concepts of the form $\tip(C)$, meant to select the ``most normal'' instances of a concept $C$.  The semantics we consider is based on rational model. But we further restrict  the semantics to minimal models, that is to say, to models that minimise the rank of domain elements.  
We show that this semantics captures exactly a notion of rational closure which is a natural extension to Description Logics of 
Lehmann and Magidor's original one.
We  also extend the notion of rational closure to the Abox component. We  provide an  \textsc{ExpTime} algorithm for computing the rational closure of an Abox and we show that it is sound and complete with respect to the minimal model semantics. 
\end{abstract}

\section{Introduction}
\vspace{-0.1cm}
Recently, in the domain of Description Logics (DLs) a large amount of work has been done in order to extend the basic formalism with nonmonotonic reasoning features. The aim of these extensions is to reason about prototypical properties of individuals or classes of individuals. In these extensions one can represent for instance knowledge expressing the fact that the heart is usually positioned in the left-hand side of the chest, with the exception of people with \emph{situs inversus}, that have the heart positioned in the right-hand side. Also, one can infer that an individual enjoys all the typical properties of the classes it belongs to. So, for instance, in the absence of information that someone has situs inversus, one would assume that it has the heart positioned in the left-hand side. A further objective of these extensions is to allow to reason about defeasibile properties and inheritance with exceptions. As another example, consider the standard penguin example, in which typical birds fly, however penguins are birds that do not fly. Nonmonotonic extensions of DLs allow to attribute to an individual the typical properties of the most specific class it belongs to. In this example, when knowing that Tweety is a bird, one would conclude that it flies, whereas when discovering that it is also a penguin, the previous inference is retracted, and the fact that Tweety does not fly is concluded.

In the literature of DLs, several proposals have appeared \cite{Straccia93,bonattilutz,baader95b,donini2002,kesattler,Casinistraccia2010,AIJ,sudafricaniKR,hitzlerdl,eiter2004,rosatiacm}.
However, finding a solution to the problem of extending DLs for reasoning about prototypical properties seems far from being solved.

In this paper, we introduce a general framework for nonmonotonic reasoning in DLs based on  (i) the use of a typicality operator \tip;
  (ii) a minimal model mechanism (in the spirit of circumscription).
The typicality operator \tip, introduced in \cite{FI09}, allows to directly express typical properties such as $\tip(\mathit{HeartPosition}) \sqsubseteq \mathit{Left}$, $\tip(\mathit{Bird}) \sqsubseteq \mathit{Fly}$, and $\tip(\mathit{Penguin}) \sqsubseteq \nott \mathit{Fly}$,
\hide{\begin{quote}
  $\tip(\mathit{HeartPosition}) \sqsubseteq \mathit{Left}$
\end{quote}
   or
\begin{quote}
$\tip(\mathit{Bird}) \sqsubseteq \mathit{Fly}$ \\
$\tip(\mathit{Penguin}) \sqsubseteq \nott \mathit{Fly}$ 
\end{quote}
}
 whose intuitive meaning is that normally, the heart is positioned in the left-hand side of the chest, that typical birds fly, whereas penguins do not.
The \tip \ operator is intended to enjoy the well-established properties of preferential semantics, described by Kraus Lehmann and Magidor (henceforth KLM)
in their seminal work 
\cite{KrausLehmannMagidor:90,whatdoes}. KLM proposed an
axiomatic approach to nonmonotonic reasoning, and individuated two systems,
preferential logic  \textbf{P} and rational logic \textbf{R}, and their corresponding semantics\hide{\footnote{Actually KLM introduced also very weak logics \Cu \ and \CL, that will be not considered here.}}. It is commonly accepted that the systems \Pe \ and \Ra \ express the core
properties of nonmonotonic reasoning. 

In \cite{AIJ,ijcai2011} nonmonotonic extensions of DLs based on the \tip \ operator have been proposed. In these extensions, the semantics of \tip \ is based on preferential logic \Pe. Nonmonotonic inference is obtained by restricting entailment to \emph{minimal models}, where minimal models are those  that minimise the truth of formulas of a special kind. In this work, we  present an alternative and more general approach. First, in our framework the semantics underlying the \tip \ operator is not fixed once for all: although we consider here only KLM's \Pe \ or \Ra \ as underlying semantics,  in principle one might choose any other underlying semantics for \tip \ based on a modal preference relation. Moreover and more importantly,  we adopt a minimal model semantics, where, as a difference with the previous approach, the notion of minimal model is completely independent from the language and  is  determined only by the relational structure of models. 

The semantic approach to nonmonotonic reasoning in  DLs presented in this work is an extension of the one  described in \cite{nmr2012} within a propositional context. We then propose a rational closure construction for DL extended with the \tip \ operator as an algorithmic  counterpart of our minimal model semantics, whenever the underlying  logic for  \tip \ is KLM logic \Ra. \emph{Rational closure} is a well-established notion introduced in \cite{whatdoes} as a nonmonotonic mechanism built on the top of \Ra \ in order to perform some further truthful nonmonotonic inferences that are not supported by \Ra \ alone. We extend it to DLs in a natural way, so that, in turn, we can see our minimal model semantics as a semantical reconstruction of rational closure.   


More in details,  we take  $\alct$  as the underlying DL and we define a nonmonotonic inference relation on the top of it 
 by restricting entailment  to minimal models: they are those ones which minimize the \emph{rank of domain elements}  by keeping
fixed the extensions of concepts and roles.
We then proceed to extend in a natural way the propositional construction of rational closure to  $\alct$ for inferring defeasible subsumptions from the TBox (TBox reasoning). Intuitively the rational closure construction amounts to assigning a \emph{rank}  (a level of exceptionality) to every concept; this rank  is  used to evaluate defeasible inclusions of the form $\tip (C) \sqsubseteq D$: the inclusion is supported by the rational closure whenever the rank of $C$ is strictly smaller than the one of $C \mint \neg D$.
Our goal is to link  the rational closure of a TBox to its  minimal model semantics, but in general it is not possible. The reason is that the  minimal model semantics is not tight enough to support the inferences provided by the rational closure. 
However we can obtain an exact corresponce between the two if we further restrict the minimal model semantics to \emph{canonical models}: these are models  that satisfy by means of a distinct element each intersection $(C_1\mint \ldots \mint C_n)$ of concepts drawn from the KB that is  satisfiable with respect to the TBox.

We then tackle the problem of extending the rational closure to ABox reasoning: we would like to ascribe defeasible properties to individuals. The  idea is to  maximise the typicality of an individual: the more is ``typical'', the more it inherits the defeasible properties of the classes it belongs too (being a typical member of them). We obtain this by  minimizing its rank (that is, its level of exceptionality), however, because of the interaction between individuals (due to roles) it is not possible to assign a unique minimal rank to each individual and  alternative minimal ranks must be considered. 
We end up with a kind of \emph{skeptical} inference with respect the ABox. We  prove  that it is sound and complete with respect to the minimal model semantics restricted to canonical models. 

The rational closure construction that we propose has not just a theoretical interest and a simple minimal model semantics, we show that it is also \emph{feasible} since its complexity is ``only'' \textsc{ExpTime} in the size of the knowledge base (and the query), thus not worse than the underlying monotonic logic. In this respect it is less complex than other approaches to nonmonotonic reasoning in DLs \cite{AIJ,bonattilutz} and comparable in complexity with the approaches in \cite{Casinistraccia2010,sudafricanisemantica,rosatiacm}, and thus a  good candidate to define  effective nonmonotonic extensions of DLs.

\section{The  operator $\tip$ and the General Semantics}\label{sez:semantica}
\vspace{-0.1cm}

Let us briefly recall the DLs $\alctp$ and $\alctr$ introduced in \cite{FI09,ecai2010DLs}, respectively. The intuitive idea is to extend the standard $\alc$ allowing concepts of the form $\tip(C)$, where $C$ does not mention $\tip$, whose intuitive meaning is that
$\tip(C)$ selects the {\em typical} instances of a concept $C$. We can therefore distinguish between the properties that
hold for all instances of concept $C$ ($C \sqsubseteq D$), and those that only hold for the typical
instances of $C$  ($\tip(C) \sqsubseteq D$) that we call  \tip-inclusions, where $C$ is a concept not mentioning $\tip$.
Formally, the language is defined as follows. 
 
 \begin{definition}\label{defelt}
 We consider an alphabet of concept names $\mathcal{C}$, of role names
$\mathcal{R}$, and of individual constants $\mathcal{O}$.
Given $A \in \mathcal{C}$ and $R \in \mathcal{R}$, we define
 $C_R:= A \tc \top \tc \bot \tc  \nott C_R \tc C_R \sqcap C_R \tc C_R \sqcup C_R \tc \forall R.C_R \tc \exists R.C_R$, and
   $C_L:= C_R \tc  \tip(C_R)$.
    A KB is a pair (TBox, ABox). TBox contains a finite set
of  concept inclusions  $C_L \sqsubseteq C_R$. ABox
contains assertions of the form $C_L(a)$ and $R(a,b)$, where $a, b \in
\mathcal{O}$.
\end{definition}

\noindent The semantics of $\alctp$ and $\alctr$  is defined respectively 
in terms of
preferential and rational\footnote{We use the expression
``rational model'' rather than ``ranked model'' which is also used
in the literature in order to avoid any confusion with the notion
of rank used in rational closure.} models:
 ordinary models of $\alc$ are equipped by a \emph{preference relation} $<$ on
the domain, whose intuitive meaning is to compare the ``typicality''
of domain elements, that is to say $x < y$ means that $x$ is more typical than
$y$. Typical members of a concept $C$, that is members of
$\tip(C)$, are the members $x$ of $C$ that are minimal with respect
to this preference relation (s.t. there is no other member of $C$
more typical than $x$). Preferential models, in which the preference relation $<$ is irreflexive and transitive,  characterize the logic $\alctp$,
whereas the more restricted class of rational models, so that $<$ is further assumed to be modular, characterizes
$\alctr$.

\begin{definition}[Semantics of $\alctp$]\label{semalctp} A model $\emme$ of $\alctp$ is any
structure $\langle \Delta, <, I \rangle$ where: $\Delta$ is the
domain;   $<$ is an irreflexive and transitive relation over
$\Delta$ that satisfies the following \emph{Smoothness Condition}:
for all $S \subseteq \Delta$, for all $x \in S$, either $x \in
Min_<(S)$ or $\exists y \in  Min_<(S)$ such that $y < x$, where
$Min_<(S)= \{u: u \in S$ and $\nexists z \in S$ s.t. $z < u \}$; $I$ is the extension function that maps each
concept $C$ to $C^I \subseteq \Delta$, and each role $R$
to  $R^I \subseteq \Delta^I \times \Delta^I$. For concepts of
$\alc$, $C^I$ is defined in the usual way. For the $\tip$ operator, we have
$(\tip(C))^I = Min_<(C^I)$.
\end{definition}

\begin{definition}[Semantics of $\alctr$]\label{semalctr} A model $\emme$ of $\alctr$ is an $\alctp$ model as in Definition \ref{semalctp} in which $<$ is
further assumed to be {\em modular}: for all $x, y, z \in \Delta$, if
$x < y$ then either $x < z$ or $z < y$. 
\end{definition}

\begin{definition}[Model satisfying a Knowledge Base]\label{Def-ModelSatTBox-ABox}
Given a model $\emme$, $I$ is extended  to assign a distinct element\footnote{We assume the well-established \emph{unique name assumption}.} $a^I$ of the domain $\Delta$ to each individual constant $a$ of $\mathcal{O}$.
$\emme$ satisfies a knowledge base $K$=(TBox,ABox), if it satisfies both its TBox
and its ABox, where: - $\emme$  satisfies TBox if for all  inclusions $C \sqsubseteq D$  in TBox, it holds $C^I \subseteq D^I$;
- $\emme$ satisfies ABox  if:
(i) for all $C(a)$  in ABox,  $a^I \in C^I$, (ii) for all $aRb$ in
ABox,  $(a^I,b^I) \in R^I$.
\end{definition}

\noindent In \cite{FI09} it has been shown that reasoning in $\alctp$ is \textsc{ExpTime} complete, that is to say adding the $\tip$ operator does not affect the complexity of the underlying DL $\alc$. We are able to extend the same result also for $\alctr$ (we omit the proof due to space limitations):

\begin{theorem}[Complexity of $\alctr$]\label{complexityalctr}
Reasoning in $\alctr$ is \textsc{ExpTime} complete.
\end{theorem}

\noindent From now on, we restrict our attention to $\alctr$ and to finite models. Given a knowledge base $K$ and an inclusion $C_L \sqsubseteq C_R$, we say that it is derivable from $K$ (we write $K \modelsalctr C_L \sqsubseteq C_R$) if $C_L^I \subseteq C_R^I$ holds in all models $\emme=\sx \Delta, <, I\dx$ satisfying $K$.

\begin{definition}\label{definition_height}
The rank $k_\emme $  of a domain element $x$ in $\emme$ is the
length of the longest chain $x_0 < \dots < x$ from $x$
to a $x_0$ such that for no ${x'}$ it holds that ${x'} < x_0$.
\end{definition}

\noindent \hide{Notice that in an $\alctr$ model $\langle \Delta, <, I \rangle$, $k_\emme$ is uniquely determined. Moreover,} Finite 
$\alctr$ models can be equivalently defined by postulating the existence of
a function $k: \Delta \rightarrow \mathbb{N}$, and then letting  $x < y$ iff
$k(x) < k(y)$.

\begin{definition}\label{definition_height_formula}
Given a model $\emme=\langle \Delta, <, I \rangle$, 
the rank $k_\emme(C_R)$ of a concept $C_R$ in $\emme$ is $i = min\{k_\emme(x):
x \in C_R^I\}$. If $C_R^I=\vuoto$, then
$C_R$ has no rank and we write $k_\emme(C_R)=\infty$.
\end{definition}

\noindent It is immediate to verify that:

\begin{proposition}\label{Truth conditions conditionals with height}
For any $\emme=\langle \Delta, <, I \rangle$, we
have that $\emme$ satisfies $\tip(C) \sqsubseteq D$ iff $k_\emme(C \sqcap D) < k_\emme(C
\sqcap \nott D)$.
\end{proposition}

\noindent As already mentioned, although the typicality operator $\tip$ itself  is nonmonotonic (i.e.
$\tip(C) \sqsubseteq D$ does not imply $\tip(C \sqcap E)
\sqsubseteq D$), the logics $\alctp$ and $\alctr$ are monotonic: what is inferred
from $K$ can still be inferred from any $K'$ with $K \subseteq K'$. In order to define a nonmonotonic
entailment we introduce the second ingredient of our minimal model semantics.
As in \cite{AIJ}, we strengthen the semantics  by
restricting entailment to a class of minimal (or preferred)
models, more precisely to models that
minimize \emph{the rank of worlds}. Informally, given two models of
$K$, one in which a given $x$ has rank 2 (because for instance
$z < y < x)$ , and another in which it has rank 1 (because only
$y < x$), we would prefer the latter,
as in this model $x$ is ``more normal'' than in the former.
We call the new logic $\alctrm$.

 Let us  define the notion of \emph{query}.
Intuitively, a query is either an inclusion relation or an assertion of the ABox, and we want to check whether it is entailed from a given KB.

\begin{definition}[Query]\label{def:query}
 A {\em query} $F$ is either an assertion $C_L(a)$ or an inclusion relation $C_L \sqsubseteq C_R$. Given a model $\emme = \langle \Delta, <, I \rangle$, a query $F=C_L(a)$ holds in $\emme$ if $a^I \in C_L^I$, whereas a query $F=C_L \sqsubseteq C_R$ holds in $\emme$ if $C_L^I \subseteq C_R^I$.
 \end{definition}

In analogy with circumscription, there are mainly two  ways of
comparing models with the same domain: 1) by keeping the valuation
function fixed (only comparing $\emme$ and $\emme'$ if $I$ and
$I'$ in the two models respectively coincide); 2) by also
comparing $\emme$ and $\emme'$ in  case  $I \neq I'$. In this work we consider
 the semantics resulting from the first
alternative, whereas we leave the study of the other one for future work (see Section \ref{sez:conclusions} below).
The semantics we introduce is a \emph{fixed interpretations minimal semantics},
for short $\mathit{FIMS}$.

\begin{definition}[$\mathit{FIMS}$]\label{Preference between models in case of fixed
valuation} Given $\emme = \langle \Delta, <, I \rangle$ and $\emme' =
\langle \Delta', <', I' \rangle$ we say that $\emme$ is preferred to
$\emme'$ \hide{with respect to the fixed interpretations minimal
semantics} ($\emme <_{\mathit{FIMS}} \emme'$) if $\Delta = \Delta'$, $I =
I'$, and for all $x \in \Delta$, $ k_\emme(x) \leq k_{\emme'}(x)$ whereas
there exists $y \in \Delta$ such that $ k_\emme(y) < k_{\emme'}(y)$. 

Given a knowledge base $K$, we say that
$\emme$ is a minimal model of $K$ with respect to $<_{\mathit{FIMS}}$ if it is a model satisfying $K$ and  there is no 
$\emme'$ model satisfying $K$ such that $\emme' <_{\mathit{FIMS}} \emme$. 
\end{definition}

\noindent Next, we extend the notion of minimal model by also taking into account the individuals named in the ABox.

\begin{definition}[Model minimally satisfying $K$]\label{model-minimally-satisfying-k}
Given $K$=(TBox,ABox), let $\emme = \langle \Delta, <, I \rangle$ and $\emme' =
\langle \Delta', <', I' \rangle$ be two models of $K$ which are minimal w.r.t. Definition \ref{Preference between models in case of fixed
valuation}. We say that $\emme$ is preferred to $\emme'$ with respect to ABox ($\emme <_{\mathit{ABox}} \emme'$) if for all individual constants $a$ occurring in ABox, $k_\emme(a^I) \leq k_\emme(a^{I'})$ and there is at least one individual constant $b$ occurring in ABox such that  $k_\emme(b^I) < k_\emme(b^{I'})$. $\emme$ minimally satisfies $K$ in case there is no $\emme'$ satisfying $K$ such that $\emme' <_{\mathit{ABox}} \emme$.
\end{definition}

\noindent We say that
$K$ minimally entails a query $F$ ($K \models_{\mathit{min}} F$) if $F$ holds in all models
 that minimally satisfy $K$.

\section{A Semantical Reconstruction of Rational Closure in DLs}\label{sez:rc}
\vspace{-0.1cm}
In this section we provide a definition of the well known rational
closure, described in \cite{whatdoes}, in the context of Description Logics.
We then provide a semantic characterization of it within the  semantics
described in the previous section.

\begin{definition}
Let $K$ be a DL knowledge base and $C$ a concept. $C$ is
said to be {\em exceptional} for $K$ iff $K \modelsalctr \tip(\top) \sqsubseteq
\neg C$.
\end{definition}

\noindent Let us now extend Lehmann and Magidor's definition of rational closure to a DL knowledge base. First, we remember that the \tip \ operator satisfies
a set of postulates that are essentially a reformulation of KLM axioms
of rational logic \Ra: in this respect, in \cite{FI09} it is shown that the \tip-assertion $\tip(C) \sqsubseteq D$ is equivalent to the conditional assertion $C \ent D$ of KLM logic \Ra.
We say that a \tip-inclusion $\tip(C) \sqsubseteq D$ is exceptional for $K$ if $C$ is exceptional for $K$. The set of \tip-inclusions  which are exceptional for $K$ will be denoted
as $\mathcal{E}(K)$. Also in this case, it is possible to define a sequence of non-increasing subsets of
$K$ $E_0 \supseteq E_1, \dots$ by letting $E_0 = K$ and, for
$i>0$, $E_i=\mathcal{E}(E_{i-1}) \unione \{ C \sqsubseteq D \in K$ s.t. $\tip$ does not occurr in $C\}$. Observe that, being $K$ finite, there is
a $n\geq 0$ such that for all $m> n, E_m = E_n$ or $E_m =
\emptyset$.

\begin{definition}\label{Def:Rank of a
formula} A concept $C$ has {\em rank} $i$ (denoted by $\rf(C)=i$) for $K$
iff $i$ is the least natural number for which $C$ is
not exceptional for $E_{i}$. {If $C$ is exceptional for all
$E_{i}$ then $\rf(C)=\infty$, and we say that $C$ has no rank.}
\end{definition}

 \noindent The notion of rank of a formula allows to define
the rational closure of the TBox of a knowledge base $K$.

\begin{definition}\label{def:rational closureDL}[Rational closure of TBox] Let $K$=(TBox,ABox) be DL knowledge base. We define the
rational closure $\overline{\mathit{TBox}}$ of TBox of $K$ where

\begin{center}
    $\mbox{$\overline{\mathit{TBox}}$}=\{\tip(C) \sqsubseteq D \tc \mbox{either} \ \rf(C) < \rf(C \sqcap \nott D)$ \\ $\mbox{or} \ \rf(C)=\infty\} \ \unione \
    \{C \sqsubseteq D \tc K \models_{\alc} C \sqsubseteq D\}$
\end{center}

\end{definition}

\noindent  It is worth noticing that Definition \ref{def:rational closureDL} takes into account the monotonic logical consequences $C \sqsubseteq D$ with respect to $\alc$. This is due to the fact that the language here is richer than that considered by Lehmann and Magidor, who only considers the set of conditionals $C \ent D$ that, as said above, correspond to \tip-inclusions $\tip(C) \sqsubseteq D$. The above Definition \ref{def:rational closureDL} also takes into account classical inclusions $C \sqsubseteq D$ that belong to our language.

In the following we show that the minimal model semantics defined in the previous section can be used to provide a semantical characterization of rational closure.

First of all, we can observe that $\mathit{FIMS}$ as it is cannot capture the rational closure of a TBox.
 For instance, consider the knowledge base $K=$(TBox,$\vuoto$) of the penguin example, where TBox contains the following inclusions: 
   $\mathit{Penguin} \sqsubseteq \mathit{Bird}$,
  $\tip(\mathit{Bird}) \sqsubseteq \mathit{Fly}$,
  $\tip(\mathit{Penguin}) \sqsubseteq \nott \mathit{Fly}$.
  We observe that $K
\not\models_{\mathit{FIMS}} \tip(\mathit{Penguin} \sqcap \mathit{Black}) \sqsubseteq \neg \mathit{Fly}$.
 Indeed in $\mathit{FIMS}$ there can be a model $\emme=\sx\Delta, <, I\dx$ in which
$\Delta = \{x,y,z\}$, $\mathit{Penguin}^I=\{x, y\}$, $\mathit{Bird}^I=\{x, y, z\}$,
$\mathit{Fly}^I=\{x, z\}$, $\mathit{Black}^I=\{x\}$, and $z < y < x$.
$\emme$ is a model of $K$, and  it is minimal with respect to
$\mathit{FIMS}$ (indeed it
is not possible to lower the rank of $x$ nor of $y$ nor of $z$
unless we falsify $K$). Furthermore,  $x$ is a typical
black penguin in $\emme$ (since there is no other black penguin preferred to
it) that flies. 
On the contrary, it can be  verified that
$\tip(\mathit{Penguin} \sqcap \mathit{Black}) \sqsubseteq \neg \mathit{Fly} \in \overline{\mathit{TBox}}$.
 Things change if we
consider $\mathit{FIMS}$ applied to models that contain a distinct domain element for {\em each
combination of concepts consistent with $K$}.
We call these models {\em canonical models}.
In the example, 
if we restrict our attention to models $\emme=\sx\Delta, <, I\dx$ that also contain a $w \in \Delta$ which is a black penguin
that does not fly, that is to say 
$w \in \mathit{Penguin}^I$, $w \in \mathit{Bird}^I$,  $w \in \mathit{Black}^I$, and $w \not\in \mathit{Fly}^I$ and can therefore be assumed to be a typical
penguin, we are able to conclude that typically black penguins do
not fly, as in rational closure. Indeed, in all minimal models of
$K$ that also contain $w$ with $w \in \mathit{Penguin}^I$, $w \in \mathit{Bird}^I$,  $w \in \mathit{Black}^I$, and $w \not\in \mathit{Fly}^I$, it holds that 
$\tip(\mathit{Penguin} \sqcap \mathit{Black}) \sqsubseteq \neg \mathit{Fly}$.

From now on, we restrict our attention to \emph{canonical minimal models}. 


 Given a knowledge base $K$ and a query $F$, we call $\lingconc$ the set of all concepts occurring (even as subconcepts) either in $K$ or in $F$, as well as of their complements. 
In order to define canonical minimal models, we consider the set of all consistent sets of concepts that are consistent with $K$. A set of concepts $\{C_1, C_2, \dots, C_n\} \subseteq \lingconc$ is consistent with $K$ if $K \not\models_{\mathcal{ALC}} C_1 \sqcap C_2 \sqcap \dots \sqcap C_n \sqsubseteq \bot$.

\begin{definition}[Canonical minimal model w.r.t. $K$ and $F$]
Given $K$ and a query $F$, a minimal model $\emme=\sx \Delta, <, I \dx$ satisfying $K$ is said to be canonical w.r.t. $K$ and $F$ if it contains at least a distinct domain element $x \in \Delta$ s.t. $x \in C^I$ for each combination $C$ in $\mathcal{S}$ consistent with $K$.
\end{definition}

\noindent We can prove the following results:

\begin{proposition}\label{proposition_rank}
Let $\emme$ be a minimal canonical model of $K$. For all concepts $C \in \lingconc$, it holds that $rank(C) = k_{\emme}(C)$.
\end{proposition}

\noindent The proof can be done by induction on the rank of  concept $C$.

\begin{theorem}\label{Theorem_RC_TBox}
Given $K$, we have that $C \sqsubseteq D \in$ $\overline{\mathit{TBox}}$ if and only if $C \sqsubseteq D$ holds in all canonical minimal models with respect to $K$ and $C \sqsubseteq D$.
\end{theorem}

\noindent This thoerem directly follows from Proposition \ref{proposition_rank}. Due to space limitations we omit the proofs.

\section{Rational Closure Over the ABox}
\vspace{-0.1cm}
In this section we extend the notion of rational closure defined in the
previous section in order to take into account the individual constants in the
ABox. We therefore address the question: what does the rational closure of a
knowledge base $K$ allow us to infer about a specific individual constant $a$
occurring in the ABox of $K$?
We propose the algorithm below to answer this question and we show that it corresponds to what is entailed by the minimal model semantics presented in the previous section. The idea of the
algorithm is that of considering all the possible minimal consistent
assignments of ranks to the individuals explicitly named in the ABox. Each
assignment adds some properties to named individuals which can be used to infer
new conclusions. We adopt a skeptical view of considering only those
conclusions which hold for all assignments. The equivalence with the semantics
shows that the minimal entailment captures a skeptical approach when reasoning
about the ABox.

\begin{definition}[Rational closure of ABox]\label{Rational Closure ABox}
$\bullet$ Let $a_1, \dots, a_m$ be the individuals explicitly named in the ABox. Let
$k_1, k_2, \dots, k_h$ all the possible rank assignments (ranging from $1$
to $n$) to the individuals occurring in ABox.

\noindent $\bullet$ We find the consistent $k_j$ with ($\overline{\mathit{TBox}}$,
ABox), where:

\noindent - for all $a_i$ in ABox, we define $\mu^j_{i} = \{ (\neg C \sqcup D)(a_i)$
s.t. $C, D \in \lingconc$, $\tip(C) \sqsubseteq D$ in
$\overline{\mathit{TBox}}$, and  $k_j(a_i) \leq rank(C)  \} \cup \{ ( \neg C
\sqcup D)(a_i) $ s.t. $C \sqsubseteq D$ in TBox $\}$;

\noindent - let $\mu^j = \mu^j_{1} \cup \dots \cup \mu^j_{n} $ for all $\mu^j_{1}
\dots \mu^j_{n} $ just calculated;

\noindent - $k_j$ is \emph{consistent} with ($\overline{\mathit{TBox}}$, ABox) if 
ABox $\cup \mu^j$ is consistent in $\alc$.

\noindent $\bullet$ We consider the \emph{minimal consistent} $k_j$ i.e. those for which
there is no $k_i$ consistent wih ($\overline{\mathit{TBox}}$, ABox) s.t. for all $a_i$, $k_i(a_i) \leq k_j(a_i)$ and for a $b$,
$k_i(b) < k_j(b)$.

\noindent $\bullet$ We define the rational closure of ABox, denoted as
$\overline{\mathit{ABox}}$, the set of all assertions derivable in $\alc$ from
ABox $\cup \mu^j$ for all minimal consistent rank assignments $k_j$, i.e:
\begin{center}
$\overline{\mathit{ABox}} = \bigcap_{k_j}\{C(a): \;$ ABox $\cup \mu^j
\models_{\alc} C(a) \}$
\end{center}
\end{definition}

\begin{theorem}[Soundness of $\overline{\mathit{ABox}}$
]\label{Theorem_Soundness_ABox}
Given $K$=(TBox, ABox), for all $a$ individual constant in ABox, we have that
if $C(a) \in$ $\overline{\mathit{ABox}}$ then $C(a)$ holds in all minimal
canonical models of $K$.
\end{theorem}
\begin{proof}
[Fact 0] For any minimal canonical model $\emme$ of $K$= (TBox, ABox) there is
a minimal rank assignment $k_j$ consistent with respect to
($\overline{\mathit{TBox}}$, ABox), such that for all $a$ in ABox and all $C$:
if  ABox $\cup \mu^j \models_{\alc} C(a)$ then $C(a)$ holds in $\emme$. This
can be proven as follows. Let $\emme$ be a minimal canonical model of $K$. Let
$k_j$ be the rank assignment corresponding to $\emme$: s.t. for all $a_i$ in
ABox $k_j(a_i)= k_{\emme}({a_i}^I)$. Obviously $k_j$ is minimal. 
Furthermore, $\emme \models $ ABox $\cup \mu^j $ . Indeed, $\emme \models $
ABox by hypothesis. To show that $\emme \models \mu^j$ we reason as follows:
for all $a_i$ let $(\neg C \sqcup D)(a_i) \in \mu^j_{i}$. If ${a_i}^I \in (\neg
C)^I$ clearly $(\neg C \sqcup D)(a_i) $ holds in $\emme$. 
On the other hand, if ${a_i}^I \in (C)^I$: by hypothesis $rank(C) \geq
k_j(a_i)$ hence
 by the correspondence between rank of a formula in the rational closure and
  in minimal canonical models (see Proposition \ref{proposition_rank}) 
also $k_{\emme}(C) \geq k_{\emme}({a_i}^I)$, but since ${a_i}^I \in (C)^I$,
$k_{\emme}(C) = k_{\emme}({a_i}^I)$, 
therefore ${a_i}^I \in (\tip(C))^I$. 
By definition of $ \mu_{i}$, and since by Theorem \ref{Theorem_RC_TBox}, 
$\emme \models \overline{\mathit{TBox}}$,  $D(a_i)$ holds in $\emme$ and
therefore also ${a_i}^I \in (\neg C \sqcup D)^I $. Hence, if  ABox $\cup \mu^j
\models_{\alc} C(a_i)$ then $C(a_i)$ holds in $\emme$.

Let $C(a) \in $
$\overline{\mathit{ABox}}$, and suppose for a contradiction that there is a
minimal canonical model $\emme$ of $K$ s.t. $C(a)$ does not hold in
$\emme$. By Fact 0 there must be a $k_j$ s.t.  ABox $\cup \mu^j
\not\models_{\alc} C(a)$, but this contradicts the fact that $C(a) \in $
$\overline{\mathit{ABox}}$. Therefore $C(a)$ must hold in all minimal canonical
models of $K$.
\hfill $\blacksquare$
\end{proof}

\begin{theorem}[Completeness of $\overline{\mathit{ABox}}$]\label{Theorem_Completeness_ABox}
Given $K$=(TBox, ABox), for all $a$ individual constant in ABox, we have that
if $C(a)$ holds in all minimal canonical models of $K$ then
$C(a) \in$ $\overline{\mathit{ABox}}$.
\end{theorem}
\begin{proof}
We show the contrapositive.
Suppose $C(a) \not\in $ $\overline{\mathit{ABox}}$, i.e. there is a minimal
$k_j$ consistent with ($\overline{\mathit{TBox}}$, ABox) s.t. ABox $\cup \mu^j
\not\models_{\alc} C(a)$.
We build a minimal canonical model $\emme = \langle \Delta, < I \rangle$ of $K$
such that $C(a_i)$ does not hold in $\emme$ as follows.
Let $\Delta = \Delta_0 \cup \Delta_1$ where $\Delta_0=\{\{C_1, \dots C_k \}
\subseteq \mathcal{S}: \{C_1, \dots C_k \}$ is maximal and consistent with $K
\} $ and $\Delta_1= \{ a_i:$ $a_i$ in ABox $\}$.
We define the rank $k_{\emme}$ of each domain element as follows: $k_{\emme}(
\{C_1, \dots C_k \}) = rank(C_1 \sqcap \dots \sqcap C_k) $, and $k_{\emme}(a_i)
= k_j(a_i)$. We then define $<$ in the obvious way: $x < y$ iff $k_{\emme}(x) <
k_{\emme}(y)$.

We then define $I$ as follows. First for all $a_i$ in ABox we let $a_i^I =
a_i$. For the interpretation of concepts we reason in two different ways for
$\Delta_0$ and $\Delta_1$. For $\Delta_0$, for all atomic concepts $C'$, we let
$\{C_1, \dots, C_k \} \in C'^I$ iff $C' \in \{C_1, \dots, C_k \} $. $I$ then
extends to boolean combinations of concepts in the usual way. It can be easily
shown that for any boolean combination of concepts $C'$, $\{C_1, \dots, C_k \}
\in C'^I$ iff $C' \in \{C_1, \dots, C_k \} $. For $\Delta_1$, we start by
considering a model $\emme' = \langle \Delta', <, I' \rangle$ such that
$\emme' \models $ ABox $\cup \mu^j$ and $\emme' \not\models C(a)$. This model
exists by hypothesis. For all atomic concepts $C'$, we let $a_i \in {C'}^I$ in
$\emme$ iff $(a_i)^{I'} \in {C'}^{I'}$ in $\emme'$. Of course for any boolean
combination of concepts $C'$, $(a_i) \in {C'}^{I}$ iff $(a_i)^{I'} \in
{C'}^{I'}$.

In order to conclude the model's construction, for each role $R$, we define
$R^I$ as follows. For $X,Y \in \Delta_0$, $(X,Y) \in R^I$ iff $\{C'$: $\forall
R.C' \in X \} \subseteq Y$. For $a_i, a_j \in \Delta_1$,
$(a_i, a_j) \in R^I$ iff
$((a_i)^{I'}, (a_j)^{I'}) \in R^{I'}$ in $\emme'$. For $a_i \in \Delta_1$, $X
\in \Delta_0$, $(a_i, X) \in R^I$ iff there is an $x \in \Delta'$ of $\emme'$
such that $(a_i^{I'}, x) \in R^{I'}$ in $\emme'$ and, for all concepts $C'$, we have $x
\in C'^{I'}$ iff $X \in C'^{I}$. $I$ is extended to quantified concepts in the
usual way. It can be shown that for all $X \in \Delta_0$ for all (possibly)
quantified $C'$, $X \in (C')^I$ iff $C' \in X$, and that for all $a_i $ in
$\Delta_1$, for all quantified $C'$, $a_i \in (C')^I$ iff $a_i \in (C')^{I'}$.

$\emme$ satisfies ABox: for $a_iRa_j$ in ABox this holds by construction. For
$C'(a_i)$, this holds since $(a_i)^{I'} \in (C')^{I'} $ in $\emme'$, hence
$(a_i)^{I} \in (C')^{I} $ in $\emme$.

$\emme$ satisfies TBox: for elements $X \in \Delta_0$, this can be proven as in
Theorem \ref{Theorem_RC_TBox}.
For $\Delta_1$ this holds since it held in $\emme'$. For the inclusion $C_l
\sqsubseteq C_j$ this is obvious. For $\tip(C_l) \sqsubseteq C_j$, for all
$a_i$ we reason as follows. First of all, if $k_j(a_i) >$ rank($C_l$) then $a_i
\not \in Min_<({C_l}^I)$ and the inclusion trivially holds. On the other if
$k_j(a_i) \leq$ rank($C_l$), $(\neg C_l \sqcup C_j )(a_i) \in \mu^j$, and
therefore $(a_i)^{I'} \in (\neg C_l \sqcup C_j )^{I'}$ in $\emme'$, hence
$(a_i)^{I} \in (\neg C_l \sqcup C_j )^{I}$ in $\emme$, and we are done.

$C(a)$ does not hold in $\emme$, since it does not hold in $\emme'$.
Last, $\emme$ is minimal: if it was not so there would be $\emme' < \emme$.
However it can be  shown that we could define a $k_{j'}$ consistent with
($\overline{\mathit{TBox}} ,$ ABox) and preferred to $k_{j}$, thus
contradicting the minimality of $k_j$, against the hypothesis.
We have then built a minimal canonical model of $K$ in which $C(a)$ does not
hold. The theorem follows by contraposition.
\hfill $\blacksquare$
\end{proof}

\begin{example}
\emph{
Consider the standard  penguin example. Let $K$ = (TBox, ABox), where TBox = $\{\tip (B) \sqsubseteq F, \tip (P) \sqsubseteq \neg F, P \sqsubseteq B\}$,  and ABox = $\{P(i), B(j)\}$.}

\emph{Computing the ranking of concepts we get that $rank(B) = 0$, $rank(P) = 1$, $rank(B\sqcap \neg F) = 1$, 
$rank(P\sqcap F) = 2$. It is easy to see that a rank assignment $k_0$ with $k_0(i) = 0$ is inconsistent with  $K$ as $\mu^0_i$ would contain $(\neg P \sqcup B)(i)$ ,  $(\neg B \sqcup F)(i)$, $(\neg P \sqcup \neg F)(i)$ and $P(i)$. Thus we are left with only two ranks $k_1$ and $k_2$ with respectively $k_1(i) = 1, k_1(j) = 0$  and $k_2(i) =  k_2(j) = 1$.}

\emph{The set $\mu^1$ contains, among the others, $(\neg P \sqcup \neg F)(i)$ ,  $(\neg B \sqcup F)(j)$.
It is tedious but esay to check that $K \cup \mu^1$ is consistent and it is the only minimal consistent one (being $k_1$ preferred to $k_2$,   thus both $\neg F(i)$ and $F(j)$ belong to $\overline{\mathit{ABox}}$.}
\end{example}

\begin{example}
\emph{This example shows the need of considering multiple ranks of individual constants:
normally  computer science courses (CS) are  taught only by  academic members (A), whereas  business courses (B) are taught only by  consultants (C), consultants and academics are disjointed, this gives the following TBox:}
\emph{$\tip (CS) \sqsubseteq \forall taught. A$, $\tip (B) \sqsubseteq \forall taught. C$, $C \sqsubseteq \lnot A$. Suppose  the ABox contains: $CS(c1)$, $B(c2)$, $taught(c1,joe)$, $taught(c2,joe)$ and let $K$= (TBox, ABox).}
\emph{Computing rational closure of TBox, we get that all atomic concepts have rank 0. Any rank assignment $k_i$ with $k_i(c1) = k_i(c2)= 0$, is  inconsistent with $K$ since the respective $\mu^i$ will contain both $(\neg CS \sqcup \forall taught. A)(c1)$ and $(\neg B \sqcup \forall taught. C)(c2)$, from which both $C(joe)$ and $A(joe)$ follow, which gives an inconsistency.}

\emph{There are two minimal consistent ranks:  $k_1$, such that $k_1(joe) = 0, k_1(c1) = 0, k_1(c2) = 1$, and $k_2$, such that $k_2(joe) = 0, k_2(c1) = 1, k_2(c2) = 0$. We have that
 ABox $\cup \mu^1 \models A(joe)$ and ABox $\cup \mu^2 \models C(joe)$. According to the skeptical definition of $\overline{\mathit{ABox}}$ neither $A(joe)$,  nor $C(joe)$ belongs to $\overline{\mathit{ABox}}$, however $(A \sqcup C)(joe)$ belongs to $\overline{\mathit{ABox}}$.}
\end{example}

\vspace{0.1cm}

\noindent Let us conclude this section by estimating the complexity of computing the rational closure of the ABox:

\begin{theorem}[Complexity of rational closure over the ABox]
Given a knowledge base $K=$(TBox,ABox), an individual constant $a$ and a concept $C$, the problem of deciding whether $C(a) \in \overline{\mathit{ABox}}$ is \textsc{ExpTime}-complete.
\end{theorem}

\begin{proof}
Let $| K|$ be the size of the knowledge base $K$ and let the size of the query be $O(| K|)$.
As the number of inclusions in the knowledge base is $O(|K|)$,
then the number $n$ of non-increasing subsets $E_i$ in the construction of the rational closure is  $O(|K|)$.
Moreover, the number $k$ of named individuals in the knowledge base is  $O(|K|)$.
Hence, the number  $k^n$ of different rank assignments to individuals is such that both $k$ and $n$ are  $O(|K|)$.
Observe that $k^n = 2^{ Log \;k^n} = 2^{n Log\; k}$. Hence, $k^n$ is $O(2^{nk})$, with $n$ and $k$ linear
in $| K|$, i.e., the number of different rank assignments  is exponential in $|K|$.

To evaluate the complexity of the algorithm for computing the rational closure of the ABox, observe that:

\noindent (i) For each $j$, the number of sets $\mu_i^j$ is $k$ (which is linear in $|K|$).
The number of inclusions in each $\mu_i^j$ is $O(|K|^2)$, as the size of $\lingconc$ is $O(|K|)$ and the number of \tip-inclusions 
$\tip(C) \sqsubseteq D \in \overline{\mathit{TBox}}$, with $C, D \in \lingconc$ is $O(|K|^2)$.
Hence, the size of set $\mu^j$ is $O(|K|^3)$.

\noindent (ii) For each $k_j$, the consistency of
($\overline{\mathit{TBox}}$, ABox) can be verified by checking the consistency of ABox $\cup \mu^j$  in $\alc$, which requires
exponential time in the size of the set of formulas
ABox $\cup \mu^j$ (which, as we have seen, is polynomial in the size of $K$).
Hence, the consistency of each $k_j$ can be verified  in exponential time in the size of $K$.

\noindent (iii) The identification of the minimal assignments $k_j$ among the consistent ones requires the comparison
of each consistent assignment with each other (i.e. $k^2$ comparisons),
where each comparison between $k_j$ and $k_{j'}$ requires $k$ steps.
Hence, the identification of the minimal assignments requires $k^3$ steps.

\noindent (iv) To define the rational closure $\overline{\mathit{ABox}}$ of ABox, for each concept  $C$
occurring in $K$ or in the query (there are $O(|K|)$ many concepts),
and for each named individual $a_i$, we have to check if $C(a_i)$ is derivable
in $\alc$ from ABox $\cup \mu^j$
for all minimal consistent rank assignments $k_j$.
As the number of different  minimal consistent assignments $k_j$ is exponential in $|K|$,
this requires  an exponential number of checks, each one requiring exponential time
in the size of the knowledge base $|K|$. The cost of the overall algorithm is therefore
exponential in the size of the knowledge base.
\hfill $\blacksquare$
\end{proof}

\section{Conclusions and Related works}\label{sez:conclusions}
\vspace{-0.1cm}
We have defined a rational closure construction for the Description Logic $\alc$ extended with a tipicality operator and provided a minimal
model semantics for it based on the idea of minimizing the rank of objects in the domain, that is their level of ``untypicality''.
This semantics correspond to a natural extension to DLs of Lehmann and Magidor's notion of
rational closure. We have also extended the notion of rational closure to the ABox, by providing an algorithm for computing it that is sound and complete with respect to the minimal model semantics. Last, we have shown an \textsc{ExpTime} upper bound for the algorithm.


In future work, we will consider a further ingredient in the recipe for nonmonotonic DLs.
 In analogy with circumscription, we can
consider a stronger form of minimization where we minimize the
rank of domain elements, but \emph{we allow to vary} the 
extensions of concepts.
Nonmonotonic extensions of \emph{low complexity} DLs based on the \tip \ operator have been recently provided \cite{ijcai2011}. In future works, we aim to study the application of the proposed semantics to DLs of the $\mathcal{EL}$ and DL-Lite families, in order to define a rational closure  for low complexity DLs.

\cite{Casinistraccia2010} discusses the application of rational closure to DLs. The authors first describe a construction to compute rational closure in the context of propositional logic, then they adapt such a construction to the DL $\alc$. As \cite{Casinistraccia2010} extends to DLs a construction which, in the propositional case, is proved to be equivalent to Lehmann and Magidor's rational closure, it may be conjectured that their construction is equivalent to our definition of rational closure in Section \ref{sez:rc}, which is the natural extension of Lehmann and Magidor's definition. 
\cite{Casinistraccia2010} keeps the ABox into account, and defines closure operations over individuals. 
They introduce a consequence relation $\Vdash$ among a KB and assertions, under the requirement that the TBox is unfoldable and
 the ABox is closed under completion rules, such as, for instance, that if 
$a: \esiste R.C \in$ ABox, then both $aRb$ and $b: C$ (for some individual constant $b$) must belong to the ABox too. 
Under such restrictions they are able to define a procedure to  compute the rational closure of the ABox assuming that the individuals explicitly named are linearly ordered, and different orders determine different sets of consequences. They show that, for each order $s$, the consequence relation $\Vdash_s$ is rational and can be computed in \textsc{PSpace}.
 In a subsequent work \cite{Casinistraccia2011}, the authors introduce an approach based on the combination of rational closure and \emph{Defeasible Inheritance Networks} (INs).

The logic $\alctr$ we consider as our base language is equivalent to the logic for defeasible subsumptions in
DLs proposed by \cite{sudafricaniKR}, when considered with $\alc$
as the underlying DL. The idea underlying this approach is very similar to that of $\alctr$: some objects in the domain are more typical than
others. In the approach by \cite{sudafricaniKR}, $x$ is more
typical than $y$ if $x \geq y$. The properties of $\geq$ correspond to those of $<$ in $\alctr$. At a
syntactic level the two logics differ, so that in
\cite{sudafricaniKR} one finds the defeasible inclusions $C \
\subsud \ D$ instead of $\tip(C) \sqsubseteq D$ of $\alctr$, however it has be shown in \cite{ecai2010DLs}
that the logic of preferential subsumption can be translated into
$\alctr$ by replacing $C \ \subsud \ D$ with $\tip(C) \sqsubseteq D$.

In \cite{sudafricanisemantica} the semantics of the logic of defeasible inclusions is strenghtened by a preferential semantics. 
Intuitively, given a TBox, the authors first introduce a preference ordering $\ll$ on the class of all
subsumption relations $\subsud$ including TBox, then they define the rational closure of TBox as the most preferred relation $\subsud$ w.r.t. $\ll$, i.e. such that there is no other relation $\subsud'$ such that TBox $\incluso \subsud'$ and $\subsud' \ll \subsud$.
Furthermore, the authors describe an \textsc{ExpTime} algorithm in order to compute the rational closure of a given TBox. However, they do not address the problem of dealing with the ABox. \hide{This algorithm is essentially equivalent to our definition of rational closure for a TBox (Definition \ref{def:rational closureDL}), from which it is straightforward to conclude that the semantics $\mathit{FIMS}$ we have proposed is sound and complete also with respect to the semantics for rational closure proposed in \cite{sudafricanisemantica}, if we restrict our concern to the TBox.}
 In \cite{sudafricaniprotege} a plug-in for the Prot\'eg\'e ontology editor 
implementing the mentioned algorithm for computing the rational closure for a TBox for OWL ontologies is described.

%
%

\hide{
\section*{Acknowledgments}
The work has been partially supported by the projects ``MIUR PRIN08 LoDeN: Logiche Descrittive Nonmonotone: Complessit\'a e implementazioni'' and ``ICT4Law - ICT Converging on Law: Next Generation Services for Citizens, Enterprises, Public Administration and Policymakers''.
}

\bibliography{biblioijcai2013}


\end{document}